\documentclass[letterpaper, 10 pt, conference]{ieeeconf}  

\IEEEoverridecommandlockouts                              

\usepackage{amsmath}
\usepackage{amssymb}
\usepackage{booktabs}
\usepackage{graphicx}
\usepackage{subfigure}
\usepackage{float}
\usepackage{url}
\usepackage{multirow}
\usepackage[linesnumbered,boxed,ruled,commentsnumbered,noend]{algorithm2e}

\newtheorem{lemma}{Lemma}
\newtheorem{theorem}{Theorem}

\newtheorem{definition}{Definition}

\usepackage{xcolor}

\title{\LARGE \bf
Asynchronous Spatial-Temporal Allocation for Trajectory Planning of Heterogeneous Multi-Agent Systems
}

\author{ Yuda Chen, Haoze Dong, and  Zhongkui Li,~\IEEEmembership{Senior Member,~IEEE}
}

\begin{document}

\maketitle
\thispagestyle{empty}
\pagestyle{empty}

\begin{abstract}

To plan the trajectories of a large-scale heterogeneous swarm, 
sequentially or synchronously distributed methods usually become intractable due to the lack of global clock synchronization.
To this end, we provide a novel asynchronous spatial-temporal allocation method.
Specifically, between a pair of agents, the allocation is proposed to determine their corresponding derivable time-stamped space and can be updated in an asynchronous way, by inserting a waiting duration between two consecutive replanning steps. 
Via theoretical analysis, the inter-agent collision is proved to be avoided and the allocation ensures timely updates.
Comprehensive simulations and comparisons with five baselines validate the effectiveness of the proposed method and illustrate its improvement in completion time and moving distance.
Finally, hardware experiments are carried out, where $8$ heterogeneous unmanned ground vehicles with onboard computation navigate in cluttered scenarios with high agility. 

\end{abstract}

\vspace{0.1cm}
\section*{Supplementary Materials}

\textbf{Video:} https://youtu.be/au3fhqbySOE

\textbf{Code:} https://github.com/CYDXYYJ/ASAP

\vspace{0.1cm}
\section{Introduction}

Collision-free trajectory planning plays an essential role for a swarm of agents navigating in a shared workspace \cite{liu2024ltl}.
The simplest method is to adopt a central coordinator to decide every agent's motion \cite{Honig2018,Li2021}.
However, this centralized solution becomes unrealistic and intractable for large-scale swarms, due to the limited interaction range and unbearable computation time.
Thus, an increasing number of works, e.g., \cite{Zhou2022,Jesus2021,Rey2018} pursue distributed solutions, where the underlying agents can decide their own actions 
via local communication with others.
Related researches gradually advance from sequentially distributed methods~\cite{Richards2004} to synchronously concurrent distributed \cite{Luis2019}, and further to asynchronously distributed ones \cite{senbaslar2022}. 

For decentralized planning, the commonly-used methods are sequentially distributed ones, e.g., \cite{Zhou2022,Richards2004,Chen2015,Morgan2016,cap2015}, 
by which the participants replan their trajectories in a sequence.
As such, the runtime only linearly increases with the number of agents.
By comparison, the concurrent distributed methods, e.g., \cite{Park2022,Toumieh2022,Roya2020,Chen2022,Chen2022-2}, 
where different agents concurrently replan trajectories, can further decrease the computation complexity. 
Concurrent methods can be divided into two categories: synchronous and asynchronous.
The former requires both global clock synchronization and timetable coordination to synchronize the calculation duration.
This, nevertheless, becomes unrealistic for heterogeneous multi-agent systems, since diversified calculation time and indirect communication are inherent characteristics of those systems.

To fulfill an asynchronous motion planning, the authors in~\cite{Alonso2018} extend the reciprocal velocity obstacles methods \cite{Berg2011} to agents with kinodynamic constraints.
Nonetheless, the short time horizon in \cite{Alonso2018} implies that it is not eligible for high-agility cases.
MADER in \cite{Jesus2021} introduces a novel Check-Recheck deconfliction scheme, in which the Check period verifies whether an agent's optimized trajectory collides with those committed by other agents, while the Recheck period assesses whether an agent has received any new trajectories during the Check period.
However, each agent is only concerned with other agents' current trajectories, neglecting those upcoming ones. 
Similar problem can be found in \cite{Kota2023} and \cite{Ma2022}.
Meanwhile, the authors in \cite{senbaslar2022} suggest an asynchronous decentralized trajectory planner capable of ensuring safety despite communication delays and interruptions.
This method extends the  Buffered Voronoi Cells~\cite{Zhou2017} to the asynchronous setting by enlarging the makespan of the constraints which may over-constrain the solution space \cite{Kota2023}.
Another noteworthy work \cite{Ferranti2023} employs the asynchronous alternating direction method of multipliers (ADMM) to resolve the non-convex constraints. 
To ensure safety despite packet loss, the method inflates the collision region and continues local ADMM iterations using the latest available predictions of neighbors.
However, it lacks a formal proof of collision avoidance.
In summary, asynchronously distributed planning methods with safety guarantees and high efficiency are still in demand.

\begin{figure} [t]
	\centering
	\includegraphics[width=0.95\linewidth]{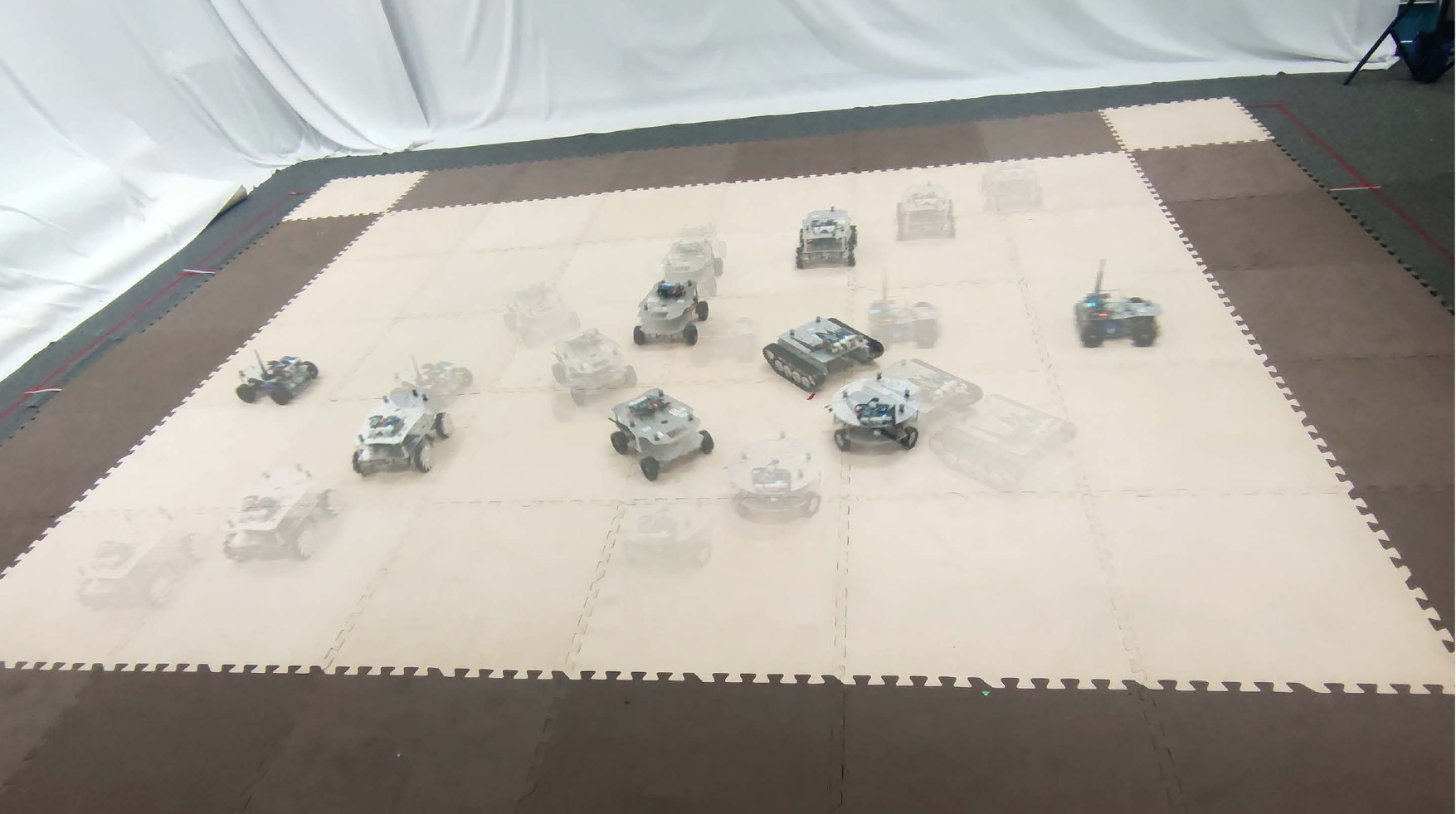}
	\caption{ Eight UGVs are crossing a crowded region. }
	\label{figure:antipodal-8}
	\vspace{-0.1in}
\end{figure}

To this end, a novel method named \textbf{A}synchronous \textbf{S}patial-\textbf{T}emporal \textbf{A}llocation (ASTA) is proposed to solve the online trajectory planning problem for heterogeneous multi-agent systems in an asynchronously distributed manner.
For a single agent, the allocation toward a neighbour is defined as a sequence of time-stamped half space.
The allocation is updated asynchronously based on the agents’ trajectories through local communication.
To achieve this, an interval called \emph{waiting time} is inserted between two consecutive replanning of each agent to allow for allocation updates.
To avoid the potential inter-agent collision when the agents adopt different versions of allocations, the update is purposely restricted to a specific duration.
It is theoretically proved that by obeying this allocation, the inter-agent collision can be avoided. 
In addition, the update frequency of the allocation is shown to have a tailored lower bound.

The contributions of the paper are summarized as follows:

\begin{itemize}
	\item[$\bullet$]
        By comparison with other asynchronously distributed methods such as \cite{Jesus2021} and \cite{Ferranti2023}, the planner can pre-determine the derivable space and concern other agents' upcoming trajectories. 
        Thus, the planned trajectory theoretically ensures inter-agent collision avoidance without the need for a recheck.	
	
	\item[$\bullet$]
	Compared to \cite{Zhou2022,Park2022}, the proposed method allows the underlying agents to replan their trajectories by merely considering their own schedules, 
	which makes it applicable for different dynamic models, e.g., double integrator, unicycle, and bicycle.
	
	\item[$\bullet$] 
	Numerical simulations and comparisons with five state-of-the-art methods \cite{Zhou2022,Jesus2021,Luis2019,Park2022,Chen2022} validate the effectiveness of our method and also illustrate the improvement in the completion time and the moving distances.
	   In the hardware experiments (Fig.~\ref{figure:antipodal-8}), $8$ heterogeneous unmanned ground vehicles (UGVs) with onboard computation finish two tasks in high agility, where the UGVs' average speed can reach up to $80\%$ of their maximum speed. 
\end{itemize}

\section{Problem Formulation}

Consider a group of $N$ agents with different kinds of dynamic models in a shared 2D workspace.
The objective is to drive these agents from their current states to their corresponding targets $x_\text{target}^i$, $i \in \mathcal{N} \triangleq \{ 1, 2, \ldots, N \}$, without \emph{any} collision under asynchronous local communication.

\subsection{Dynamic Models}\label{subsec:dynamic model}

For agent $i \in \mathcal{N}$, its state and control input at time $t$ are denoted as $x^i(t)\in \mathbb{R}^m$ and $u^i(t)\in \mathbb{R}^q$, $m,q \in \mathbb{Z}^+$, respectively.
Its dynamic model is described as follows: 
\begin{equation} \label{dynamic-model}
	\dot{x}^i(t) = f^i (x^i(t), u^i(t)), ~\ x^i(t) \in \mathcal{X}^i, \ u^i(t) \in \mathcal{U}^i,
\end{equation}
where $\mathcal{X}^i$ and $\mathcal{U}^i$ denote the set of available states and control inputs, respectively;
$f^i(\bullet)$ reflects the differential constraint. 

Note that the dynamic models of the agents in \eqref{dynamic-model} are heterogeneous, which include the commonly-encountered double integrator, unicycle and bicycle as special cases, elaborated as follows:
\begin{itemize}
	\item  Double integrator: $x=[p_x,p_y,v_x,v_y]^\mathrm{T}$, 
	$u=[a_x,a_y]^\mathrm{T}$, $f(x,u)=[v_x,v_y,a_x,a_y]^\mathrm{T}$, where 
	$\mathcal{X}=\left\{x \in \mathbb{R}^4 \ | \ |v_x|,|v_y| \leq v_\text{max} \right\}$, $\mathcal{U}=\left\{u \in \mathbb{R}^2 \ | \ |a_x|,|a_y| \leq a_\text{max} \right\}$;
	\item  Unicycle: $x=[p_x,p_y,\theta,v]^\mathrm{T}$, 
	$u=[a,\omega]^\mathrm{T}$, $f(x,u)=[v \cos \theta,v \sin \theta,\omega,a]^\mathrm{T}$, where $\mathcal{X}=\left\{x \in \mathbb{R}^4 \ | \  0 \leq v \leq v_\text{max} \right\}$, 
	$\mathcal{U}=\left\{u \in \mathbb{R}^2 \ | \ |\omega| \leq \omega_\text{max}, |a| \leq a_\text{max} \right\} $;
	\item  Bicycle: $x=[p_x,p_y,\theta,v]^\mathrm{T}$, $u=[a,\delta]^\mathrm{T}$, 
	$f(x,u)=[v \cos \theta,v \sin \theta,v \tan \delta / L,a]^\mathrm{T}$, where $L \in \mathbb{R}^+$, $\mathcal{X}=\left\{x \in \mathbb{R}^4 \ | \  0 \leq v \leq v_\text{max} \right\}$, 
	$\mathcal{U}=\left\{u \in \mathbb{R}^2 \ | \ |\delta| \leq \delta_\text{max}, |a| \leq a_\text{max} \right\} $;
\end{itemize}  
in which $p$, $v$, $a$, $\theta$, $\omega$, $\delta$ and $L$ represent the agent's position, linear velocity, acceleration, heading angle, angular velocity, steering angle and wheelbase, and the subscripts $x$, $y$ denote the $x$, $y$ coordinates, respectively.

\subsection{Inter-Agent Collision Avoidance}

In order to characterize the inter-agent collision avoidance, we first define the representation of the agents' trajectories as follows:
\begin{definition}\label{def:trajectory}
    The trajectory of agent $i$ between time span $[t_a,t_b]$ is defined as 
    \begin{equation*}
        \mathcal{T}raj^i (t_a,t_b) \triangleq \left\{  S^i \left( x^i(t) \right) \times t \ | \ t \in [t_a,t_b]  \right\}.
    \end{equation*}
    The shape of agent $i$ at time $t$ is denoted by a convex polygon $S^i \left(x^i(t)\right)$.
\end{definition}

At any time $t>0$, for any pair of agents $i$ and $j$, they are collision-free if and only if  $S^i \left(x^i(t)\right) \cap S^j \left(x^j(t)\right)=\emptyset$.
Accordingly, their trajectories are collision-free if and only if $\mathcal{T}raj^i (0,+\infty) \cap \mathcal{T}raj^j (0,+\infty)=\emptyset$.

\subsection{Problem Statement} \label{sec:problem-statement}

In this work, the trajectory planning is carried out in a receding horizon way, 
wherein $x_{n^i}^i(t)$ is the planned state at $t$ in agent $i$'s $n^i$-th replanning. 
Additionally, the planning horizon of the $n^i$-th replanning is denoted as $T_{n^i}^i \in \mathbb{R}^+$.
The time when agent $i$ finishes its $n^i$-th replanning step is denoted as $t_{n^i}^i$, which is also the starting time of the $n^i$-th trajectory. 
The $n^i$-th replanned trajectory is denoted as $\mathcal{T}raj_{n^i}^i(t_{n^i}^i, +\infty)$.
Particularly, for the time beyond the planning horizon, i.e., $t > t_{n^i}^i + T_{n^i}^i$, we enforce $x_{n^i}^i(t) = x_{n^i}^i(t_{n^i}^i + T_{n^i}^i)$.



Then, for agent $i$ in its $n^i$-th replanning, the trajectory generation problem can be formulated as the following optimal control problem (OCP):
\begin{subequations} \label{original-optimal-control}
	\begin{align} 
		\min_{x_{n^i}^i(t), u_{n^i}^i(t)} & \ C(x_{n^i}^i, u_{n^i}^i), \notag \\
		\text{s.t.} \ &\eqref{dynamic-model},  \notag \\
		&x^i(t_{n^i}^i) = \hat{x}^i(t_{n^i}^i), \label{initial-cons} \\
		&\mathcal{T}raj_{n^i}^i \cap \mathcal{T}raj^j = \emptyset, \ \forall j \neq i, \label{collision-cons} 
	\end{align}
\end{subequations}
where $t \in [t_{n^i}^i,t_{n^i}^i+T_{n^i}^i]$; $\hat{x}^i(t_{n^i}^i)$ is the predicted state at $t_{n^i}^i$; and $C(x_{n^i}^i, u_{n^i}^i)$ is the objective function, defined as
\begin{equation*}
\begin{aligned}
	&C(x_{n^i}^i, u_{n^i}^i) \\&\quad= \int_{t_{n^i}^i}^{t_{n^i}^i+T_{n^i}^i} {\delta x_{n^i}^i(\tau)}^\mathrm{T} Q \ {\delta x_{n^i}^i(\tau)} + {u_{n^i}^i(\tau)}^\mathrm{T} P \ u_{n^i}^i(\tau) d \tau, 
\end{aligned}
\end{equation*}
 which is utilized to drive agent $i$ to its target and penalize the control inputs $u_{n^i}^i(\tau)$, where 
$\delta x_{n^i}^i(\tau) = x_{n^i}^i(\tau) - x^i_\text{target}$, 
$Q$ and $P$ are the weighted matrices.
In the ideal case, the lower-level controller is assumed to perfectly follow the planned trajectory such that $\hat{x}^i(t_{n^i}^i) = \mathcal{T}raj_{n^i-1}^i(t_{n^i}^i)$.

Note that the collision avoidance constraints~\eqref{collision-cons} are impossible to be directly established at the replanning stage, since the exact future motions of other agents are unavailable. 
To solve this problem, we propose a method to dynamically reformulate constraints~\eqref{collision-cons} such that the OCP~\eqref{original-optimal-control} can be formulated and solved in an asynchronous manner while the inter-agent collision is theoretically avoided.

\section{Proposed Method}\label{section:proposed_solution}

As illustrated in Fig.~\ref{figure:component}, the proposed method has three main components: 
i) Spatial-Temporal Allocation (STA), consisting of a sequence of half spaces with time stamps $t_i$ ($i=1,2,\ldots$), is used to determine feasible time-stamped space w.r.t another agent (the yellow part).
ii) Between a pair of agents, STA is asynchronously updated based on their trajectories towards a specific duration (the gray part).
iii) Trajectory planning via STA aims to confine each agent's planned trajectory in its respective feasible time-stamped space (the red part). 
These components will be presented in the following three subsections.

\begin{figure}
	\centering
	\includegraphics[width=0.95\linewidth]{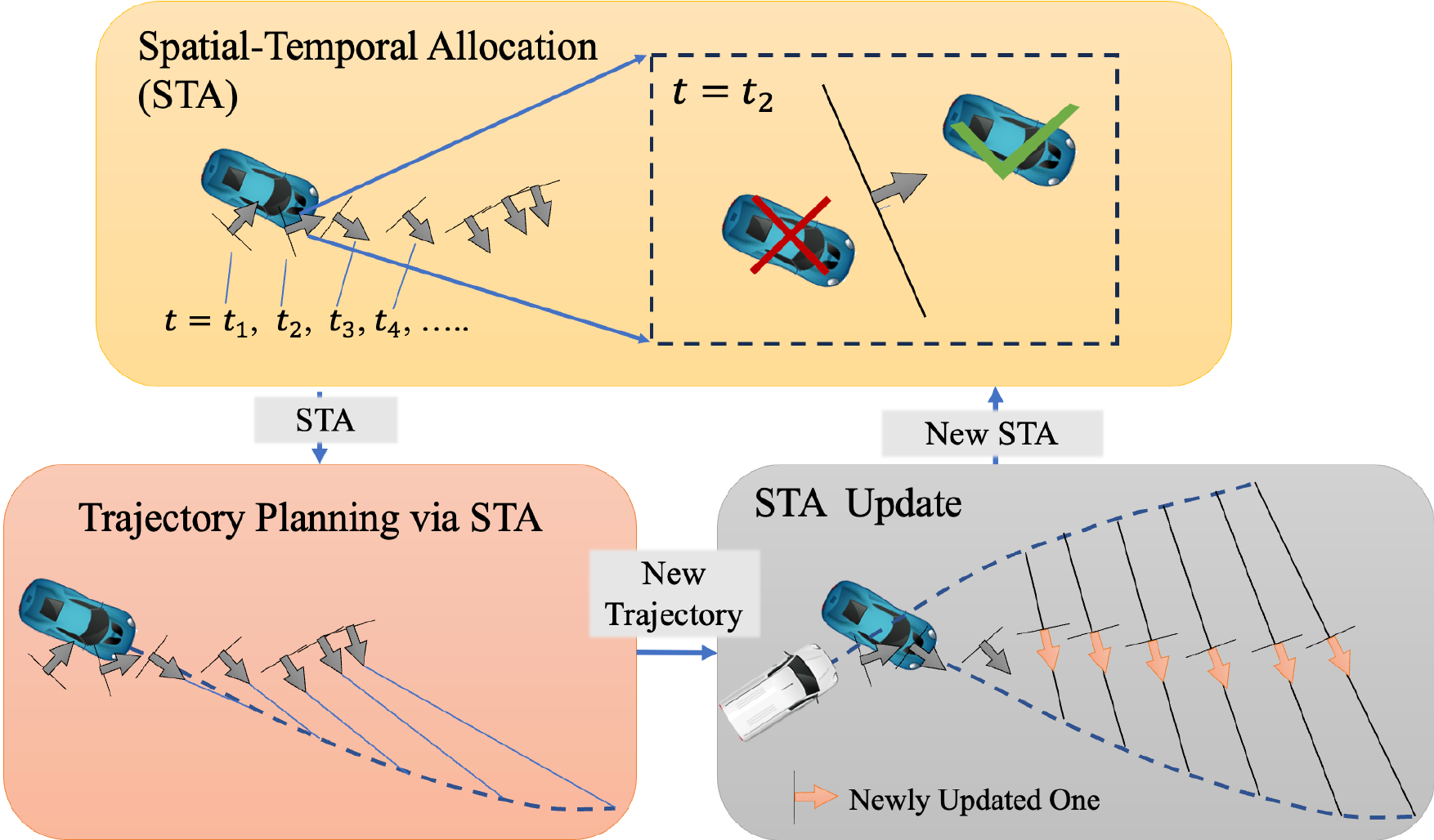}
	\caption{The components of the proposed method. A solid line and its normal vector are used to demonstrate a half space, where the solid line represents its boundary and the normal vector points toward its safe region.}
	\label{figure:component}
	\vspace{-0.4cm}
\end{figure}

\subsection{Spatial-Temporal Allocation}

STA between agents $i$ and $j$ aims to allocate the time-stamped feasible region for replanning, 
so as to avoid collision between agents.
The definition of STA is given below.

\begin{definition} \label{def:STA}
	The Spatial-Temporal Allocation (STA) of agent $i$ w.r.t. agent $j$ is 
	\begin{equation*}
		\mathbb{A}^{ij} = \left\{ \mathcal{H}^{ij}(t) \times t \ | \ t \in [t^{R_1},+\infty) \right\},
	\end{equation*}
	where $\mathcal{H}^{ij}(t) \triangleq \{p \ | \ {a^{ij}}^T (t) \ p > b^{ij}(t)\}$ represents a half space in 2D space,
    $a^{ij} (t) \in \mathbb{R}^2$ is a normal vector of the boundary of $\mathcal{H}^{ij}(t)$ pointing to the interior of $\mathcal{H}^{ij}(t)$ and $b^{ij}(t) \in \mathbb{R}$ is the offset.  We call $\mathcal{H}^{ij}(t)$ a \textbf{T}ime-\textbf{S}tamped \textbf{H}alf \textbf{S}pace (TSHS).
	Moreover, $t^{R_1}$ is the time when these agents establish communication.
	Conversely, we have $\mathcal{H}^{ji}(t) \triangleq \{p \ | \ {a^{ji}}^T (t) \ p > b^{ji}(t)\}$, where $a^{ji}(t) = -a^{ij} (t)$, $b^{ji}(t) = -b^{ij}(t)$.
	As a result, it is clear that  
	$\mathcal{H}^{ij}(t) \cap \mathcal{H}^{ji}(t) = \emptyset$ and $\mathbb{A}^{ij} \cap \mathbb{A}^{ji} = \emptyset$.
\end{definition}

Since STA is updated along with the replanning process, 
we use $\mathbb{A}_m^{ij}$ to represent the allocation after the $({m-1})$-th update in the sequel.

\subsection{Allocation Update} \label{subsection:protocol}

Allocation update is a core part in our method which keeps the collision avoidance constraints up-to-date. 
For agent $i$, there is a waiting time $T_w^i$ between two consecutive trajectory calculation times $T_c^i$ (see Fig.~\ref{protocol-process}). In this work, we particularly enforce that $T_w^i > T_c^i$, in order to ensure a lower bound of the updating frequency. 

At time $t^i_{n^i-1}$ when the $(n^i-1)$-th replanning step is finished, the newly planned trajectory $\mathcal{T}raj_{n^i-1}^i$ is broadcast to other waiting agents.
Specifically, if agent $j$ stays at waiting time after its $n^j$-th replanning, it will reach a \emph{renewal} with agent $i$ based on $\mathcal{T}raj_{n^j}^j$ and $\mathcal{T}raj_{n^i-1}^i$. 
Fig.~\ref{protocol-process} provides an illustration.
During the waiting time, agent $i$ can also receive other agents' data and reach renewals.
A renewal is defined as follows:
\begin{equation*}
    \mathcal{R}_m^{ij} \triangleq \{ \mathcal{H}^{ij}(t) \times t \ | \ t \in [t_s^{R_m},+\infty) \},
\end{equation*}
where $t_s^{R_m} \triangleq \max\left\{ t_{(n^i-1)+1}^i, t_{n^j+1}^j \right\}$ is the start time of the $m$-th renewal.
Moreover, a part of this renewal from $t_a$ to $t_b$ is denoted as $\mathcal{R}_m^{ij} (t_a,t_b) \triangleq \{ \mathcal{H}^{ij}(t) \times t \ | \ t \in [t_a,t_b) \}$. 


The process that produces a renewal via $\mathcal{T}raj_{n^i-1}^i$ and $\mathcal{T}raj_{n^j}^j$ is  to generate hyperplane in order to split the space into two halves $\mathcal{H}^{ij}(t)$ and $\mathcal{H}^{ji}(t)$ based on $S^i(x_{n^i-1}^i(t))$ and $S^j(x_{n^j}^j(t))$ for $t \in [t_s^{R_m}, +\infty)$. 
According to the separating hyperplane theorem \cite{Boyd2004}, if $S^i(x_{n^i-1}^i(t)) \cap S^j(x_{n^j}^j(t))=\emptyset$, then we can always find a hyperplane that separates those two polygons. 
One of the possible implementations of separating hyperplane is using GJK algorithm \cite{Gilbert1988,Park2022}.


\begin{figure}
	\centering
	\includegraphics[width=0.95\linewidth]{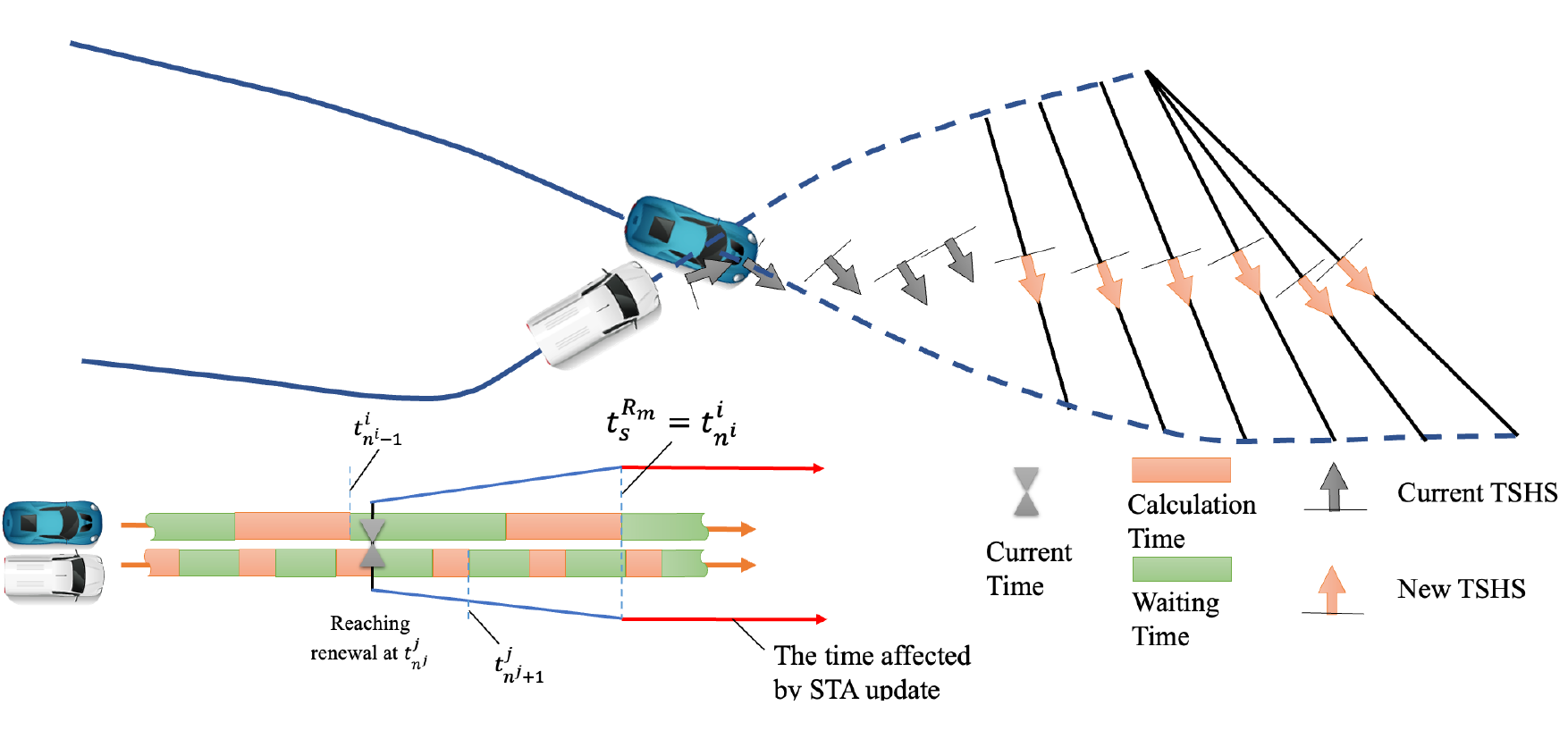}
	\caption{A demonstration about the STA update for two cars. Both cars stay at waiting time and reach a new renewal, and the TSHS beyond $t_s^{R_m}$ are regenerated to adapt the new situation. }
	\label{protocol-process}
	\vspace{-0.2in}
\end{figure}


Notably, despite the fact that the duration concerned is $[t_s^{R_m}, +\infty)$ for the $m$-th renewal, 
only the duration $[t_s^{R_m}, t_e^{R_m} )$, where $t_e^{R_m} \triangleq \max \left\{ t_{n^i-1}^i+T_{n^i-1}^i, t_{n^j}^j+T_{n^j}^j \right\}$, needs to be considered.
Because for $t > t_e^{R_m}$, $S^i(x_{n^i-1}^i(t)) = S^i(x_{n^i-1}^i(t_e^{R_m}))$ and $S^j(x_{n^j}^j(t)) = S^j(x_{n^j}^j(t_e^{R_m}))$ hold according to the assumption in Section~\ref{sec:problem-statement} that for $t > t_{n^i}^i + T_{n^i}^i$, $x_{n^i}^i(t) = x_{n^i}^i(t_{n^i}^i + T_{n^i}^i)$.
Thus, we have $\mathcal{H}^{ij}(t) = \mathcal{H}^{ij}(t_e^{R_m})$ as well as 
$\mathcal{H}^{ji}(t) = \mathcal{H}^{ji}(t_e^{R_m})$ for $t > t_e^{R_m}$.

For STA, when the first renewal starts, it is initialized as $\mathbb{A}_1^{ij} = \mathcal{R}_1^{ij}(t^{R_1},+\infty)$.
In particular, for the first renewal, we enforce $t_s^{R_1} = t^{R_1}$ to avoid collision during $ [t^{R_1},t_s^{R_1}]$.
Then, for a new renewal $\mathcal{R}_m^{ij}$, the STA is updated as follows:
\begin{equation} \label{protocol-update}
	\mathbb{A}_m^{ij} = \mathbb{A}_{m-1}^{ij}(t^{R_1},t_s^{R_m}) \cup \mathcal{R}_m^{ij}(t_s^{R_m},+\infty),
\end{equation}
where
$
	\mathbb{A}_{m-1}^{ij}(t^{R_1},t_s^{R_m}) = \left\{ \mathcal{H}^{ij}(t) \times t \ | \ t \in [t^{R_1},t_s^{R_m}) \right\}.
$
In this updating mechanism, STA is only updated for the time beyond $t_s^{R_m}$, which can also be   reflected in Fig.~\ref{protocol-process}.

A pair of agents reach their renewal if and only if one of them just finishes its replanning and the other one stays at waiting time. 
This is a typical asynchronous method and thus called \emph{Asynchronous} Spatial-Temporal Allocation.

\subsection{Trajectory Planning Using STA}\label{subsec:Trajectory_planning}

In OCP~\eqref{original-optimal-control}, during agent $i$'s $n^i$-th replanning,
given $\mathbb{A}_m^{ij}$ as the latest STA with agent $j$,
the constraints~\eqref{collision-cons} can be replaced by 
\begin{equation} \label{original-obey-protocol}
	\mathcal{T}raj_{n^i}^i \subset \mathbb{A}_m^{ij}, \ j \in \mathcal{N}^i, 
\end{equation}
where $\mathcal{N}^i$ is the set of agents that have ever established communication with agent $i$.
Based on this replacement, 
we further reformulate the OCP~\eqref{original-optimal-control}:
\begin{subequations} \label{optimal-control}
	\begin{align}
		\min_{x_{n^i}^i(t), u_{n^i}^i(t)} & \ C(x_{n^i}^i, u_{n^i}^i) \notag \\
		\text{s.t.} \ &\eqref{dynamic-model}, \eqref{initial-cons}, \notag \\
		&S^i(x_{n^i}^i(t)) \in \mathcal{H}^{ij} (t), \label{finite-obey-protocl-1} \\
		&S^i(x_{n^i}^i(t_{n^i}^i+T_{n^i}^i)) \in \mathcal{H}^{ij} (t^\prime), \label{finite-obey-protocl-2}
	\end{align}
\end{subequations}
where $t \in [t_{n^i}^i,t_{n^i}^i+T_{n^i}^i]$ and $t^\prime \in (t_{n^i}^i+T_{n^i}^i,t_e^{R_m}]$.  
Note that the constraints \eqref{original-obey-protocol} consider an intractably infinite time interval $t \in [t_{n^i}^i, +\infty)$.
We replace it with constraints~\eqref{finite-obey-protocl-1} and \eqref{finite-obey-protocl-2} which only consider a finite horizon, thus facilitating the application.
The reason why OCP~\eqref{original-optimal-control} with constraints~\eqref{original-obey-protocol} is equal to OCP \eqref{optimal-control} can be found in the Lemma~\ref{lemma:finite-horizon} of Appendix~\ref{sec:appendix-1}.
In real implementation, the underlying continuous OCP will be reformulated and resolved in a discrete form, which will be stated in Section~\ref{implementation}.

\subsection{The Overall Algorithm}

\begin{algorithm}[t] \label{AL:ASAP}
	\caption{The Complete Algorithm}\label{algorithm}
	\SetKwInOut{Input}{Input}
	\Input{$x^i(0)$, $x^i_\text{target}$}
	$\mathcal{T}raj_{0}^i = \left\{  S \left( x^i(0) \right) \times t \ | \ t \in [0,+\infty)  \right\}$ \label{algline:initial}\;
	\While{\rm{agent $i$ not reaching target} \label{algline:impc-while}}
	{
		$\mathcal{N}^i \leftarrow$ scan communication network \label{algline:scan}\;
		send $\mathcal{T}raj^i_{n^i-1}$ to agent $j \in \mathcal{N}^i$ \label{algline:intiate-agreement}\;
		\For{ $j \in \mathcal{N}^i$ concurrently}{
			open receiver in a new thread \label{algline:open-receiver}\;
			update allocation as \eqref{protocol-update} if receiving feedback \label{algline:update-protocol}\;
			close receiver \label{algline:close-receiver}\;
		}
		$x^i_{n^i} \leftarrow$ Trajectory planning as \eqref{optimal-control} \label{algline:OCP}\;
		send $x^i_{n^i}$ to the lower-level controller\label{algline:execute}\;
	} 
\end{algorithm}   

The overall planning algorithm built up by ASTA is outlined in Alg.~\ref{AL:ASAP}.
At the beginning, 
each agent initializes its trajectory (Line~\ref{algline:initial}), where we enforce that $x^i(0) \in \mathcal{X}^i_e \triangleq \left\{ x \ | \ \exists u, f^i (x, u)=0 \right\}$.
Afterwards, the agent scans its communication network to find all neighbors (Line~\ref{algline:scan}) and then broadcasts its  trajectory via communication network (Line~\ref{algline:intiate-agreement}).
Then, for each neighbor, 
it opens an independent thread to wait for other agents' feedback (Line~\ref{algline:open-receiver}).
Once receiving a confirmation signal, the allocation will be updated as in equation~\eqref{protocol-update}.
Thereafter, the receivers regarding all neighbors are closed (Line~\ref{algline:close-receiver}) and the new trajectory is planned according to OCP~\eqref{optimal-control}.
However, in this work, the feasibility of the OCP cannot be guaranteed, which highlights an area for our future work.
To tackle this, if the given OCP~\eqref{optimal-control} is infeasible, 
the previous trajectory will be continuously adopted.
Finally, the trajectory is sent to the lower-level controller, which controls the agent's motion (Line~\ref{algline:execute}).
The whole loop will be carried out indefinitely until this agent reaches its target position.

\subsection{Algorithm Analysis} \label{subsec:analysis}

Note that there may happen potential collision between a pair of agents due to conflicts in their safety constraints, since they may adopt different versions of allocations.
To avoid this, the proposed method is designed to partially update STA after a specific time $t_s^{R_m}$ as \eqref{protocol-update}.
Thanks to such a design, inter-agent collision can be theoretically avoided, which is supported by the following theorem.

\begin{theorem} \label{theorem:safety}
	For a pair of agents $i$ and $j$, if i) agents $i$'s and $j$'s trajectories do not collide with each other when they establish their allocation, and ii) in every following replanning step, they obey their corresponding allocations, then $\forall t > t^{R_1}$, they will not collide with each other.
\end{theorem}

\begin{proof} 
	Please refer to Appendix~\ref{sec:appendix-2}.
\end{proof}

Besides safety guarantee, another key point is that as the allocation is updated in a triggered way, a lower bound of updating frequency is required.
Otherwise, due to the out-of-date allocation, the efficiency of trajectory planning of the swarm will be adversely affected.
Therefore, we enforce that the waiting time $T_w^i$ is longer than the computation time $T_c^i$ in Section~\ref{subsection:protocol}.
Thus, the following property can be obtained.

\begin{theorem} \label{theorem:lower-bound-frequency}
	Between agents $i$ and $j$, 
	the lower bound of the frequency of the allocation update is 
    \vspace{-0.05in}
    \begin{equation*}
        \mathcal{F} = \left ( \min \left\{  T_c^i,  T_c^j \right\} + \max \left\{ T_c^i + T_w^i, T_c^j + T_w^j \right\} \right )^{-1}.
    \end{equation*}
\end{theorem}
\vspace{0.05in}
\begin{proof}
	Please refer to Appendix~\ref{sec:appendix-3}.
\end{proof}

\section{Simulations and Experiments}

In this section, we first illustrate the implementation of the proposed method and
then validate its performance via numerical simulations and hardware experiments.

\subsection{Algorithm Implementation} \label{implementation}

This subsection demonstrates how to implement the proposed planning algorithm on digital platforms.  


Due to the fact that the OCP~\eqref{optimal-control} is characterized in the continuous-time manner, it requires to be reformulated as a numerical optimization problem via discretization based on a sampling time interval $h^i \in \mathbb{R}$.
Accordingly, the length of horizon $K^i$ will be determined as $K^i=\lfloor \frac{T^i}{h^i} \rfloor$.
Consequently, in the $n^i$-th replanning step, the planned state $x_{n^i}^i(t_{n^i}^i+k h^i)$ at time $t_{n^i}^i+k h^i$ and the control input $u_{n^i}^i(t_{n^i}^i+(k-1) h^i)$ at time $t_{n^i}^i+ (k-1) h^i$, where 
$k \in \mathcal{K}^i \triangleq \{ 1,2,\ldots,K^i \}$
are chosen as the optimization variables. 
Towards the constraints \eqref{finite-obey-protocl-1} and \eqref{finite-obey-protocl-2}, they can be rewritten as 
$ {p_v^i}^\mathrm{T} a^{ij}(t_{n^i}^i+k h^i) > b^{ij}(t_{n^i}^i+k h^i)$, where 
$k \in \mathcal{K}^i$, $p_v^i \in V_{n^i}^i(t_{n^i}^i+k h^i)$, $V_{n^i}^i(t)$ is the collection of vertices of polygon regarding the planned state $x_{n^i}^i(t+k h^i)$,
$a^{ij}(t_{n^i}^i+k h^i)$ and $b^{ij}(t_{n^i}^i+k h^i)$ represent the hyperplane obtained at time $t_{n^i}^i+k h^i$.
Other constraints can be similarly discretized at these time steps.




In terms of implementation, an agent only needs to consider its neighbors within a specified radius. 
This reduces the number of optimization constraints and achieves manageable computational complexity and admissible scalability.
The obstacle avoidance scheme from our previous work \cite{Chen2022-2} can be introduced to apply the proposed method in obstacle environments.
Furthermore, in our implementation, an agent can determine the relative time difference between each other via communication. 
Thus, a global clock synchronization is eliminated.
Lastly, we adopt the deadlock resolution scheme in \cite{Chen2022} to prevent deadlock.

\subsection{Numerical Simulations}

In this subsection, the OCP discretized as in Section~\ref{implementation} is formulated and resolved by acados \cite{Verschueren2021}.
In addition, the proposed communication mechanism is realized via Robot Operating System (ROS).
All the simulation cases are run on a single computer with an Intel Core I9@3.5GHz, utilizing multiple programs.

\begin{table} [t]
	\caption{Agent's information and simulation results. ($T_c^i[{\rm s}]$: calculation time. $T_w^i[{\rm s}]$: waiting time. $h^i[{\rm s}]$: sampling time. $K^i$: length of horizon. $T_\text{cost}[{\rm ms}]$: maximum/average replanning runtime. $L[{\rm m}]$: transition length. $T_t[{\rm s}]$: moving time.) }
	\label{table:agent}
	\begin{tabular}{ccccccccc} 
		\toprule
		Index & Model   & $T_c^i$ & $T_w^i$ & $h^i$ & $K^i$ & $T_\text{cost}$ & $L$ &$T_t$\\
		\midrule
		1 & Bicycle  & 0.07 & 0.09 & 0.15 & 20 & 23/11 & 4.2 & 4.5\\
		2 & Double.  & 0.12 & 0.14 & 0.15 & 20 & 17/9 & 4.4 & 8.8\\
		3 & Unicycle  & 0.16 & 0.21 & 0.15 & 20 & 19/11 & 4.4 & 7.7\\
		4 & Double.  & 0.10 & 0.17 & 0.15 & 20 & 19/10 & 4.3 & 7.4\\
		5 & Bicycle  & 0.08 & 0.10 & 0.15 & 20 & 22/11 & 4.3 & 6.0\\
		6 & Double.  & 0.10 & 0.14 & 0.15 & 20 & 17/9 & 4.8 & 7.3\\
		7 & Unicycle  & 0.12  & 0.16 & 0.15 & 20 & 26/13 & 4.4 & 7.7\\
		8 & Unicycle  & 0.16   & 0.18 & 0.15 & 20 & 19/10 & 4.6 & 8.9\\
		\bottomrule
	\end{tabular}
    \vspace{-0.1in}
\end{table} 

The first simulation example is that $8$ agents navigate to their antipodal position in a circle with a diameter of $4.0$m.
The detailed information about the underlying agents and the simulation results are listed in Table~\ref{table:agent}.
Their diameters are uniformly set as $0.4$m, but their maximum speeds range from $0.6{\rm m/s}$ to $1.0{\rm m/s}$.
In this simulation,  as elaborated in Section~\ref{subsec:dynamic model}, agents with different dynamic models are considered.
The agents' trajectories and inter-agent distance are depicted in Fig.~\ref{figure:smiu1}, from which it can be observed that 
the minimum inter-agent distance is around $0.5$m, larger than the safety distance $0.4 \rm{m}$. It can be seen that safety is guaranteed during the simulation.
This task is finished within $8.9$s and the maximum transition length is $4.8$m as shown in Table~\ref{table:agent}.


Three more complicated scenarios are further simulated.
In the first scenario of Fig.~\ref{figure:16agent}, we consider a system composed of $16$ agents, confirming its effectiveness in handling large-scale systems.
In its second scenario, $8$ UGVs with bicycle model exchange their lateral positions concurrently, which simulates complicated lane changes in urban road traffic.
The result indicates a relatively smooth and fast position exchange, which further exhibits the practicability of our method in real traffic management.
To further evaluate the proposed method in an obstacle environment, we conducted a simulation as shown in Fig.~\ref{figure:8_in_obstacle}. The result indicates that the method can be extended to handle asynchronous trajectory planning in obstacle scenarios while ensuring safety.

\begin{figure}
	\centering
	\includegraphics[width=0.58\linewidth]{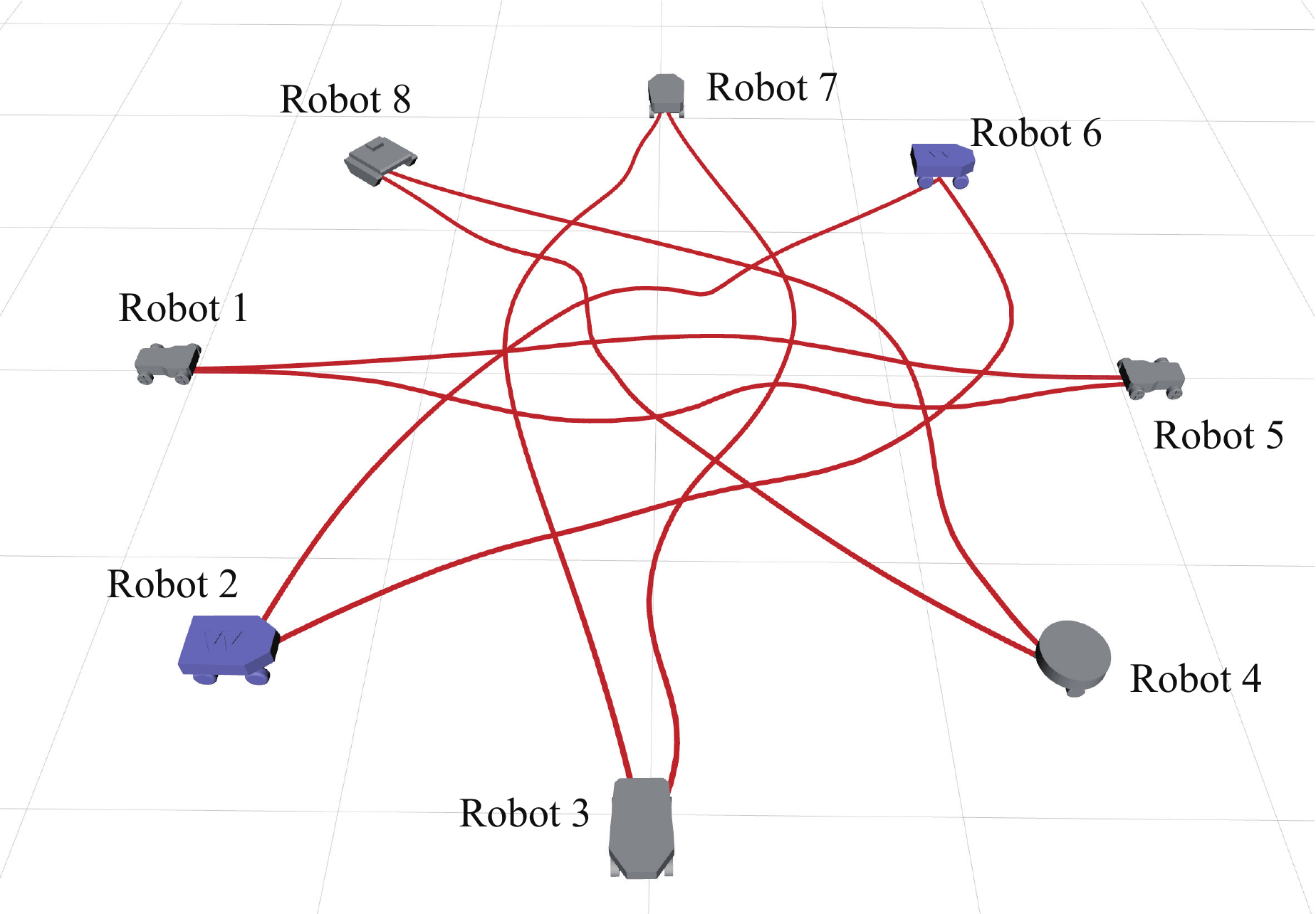}
	\includegraphics[width=0.4\linewidth]{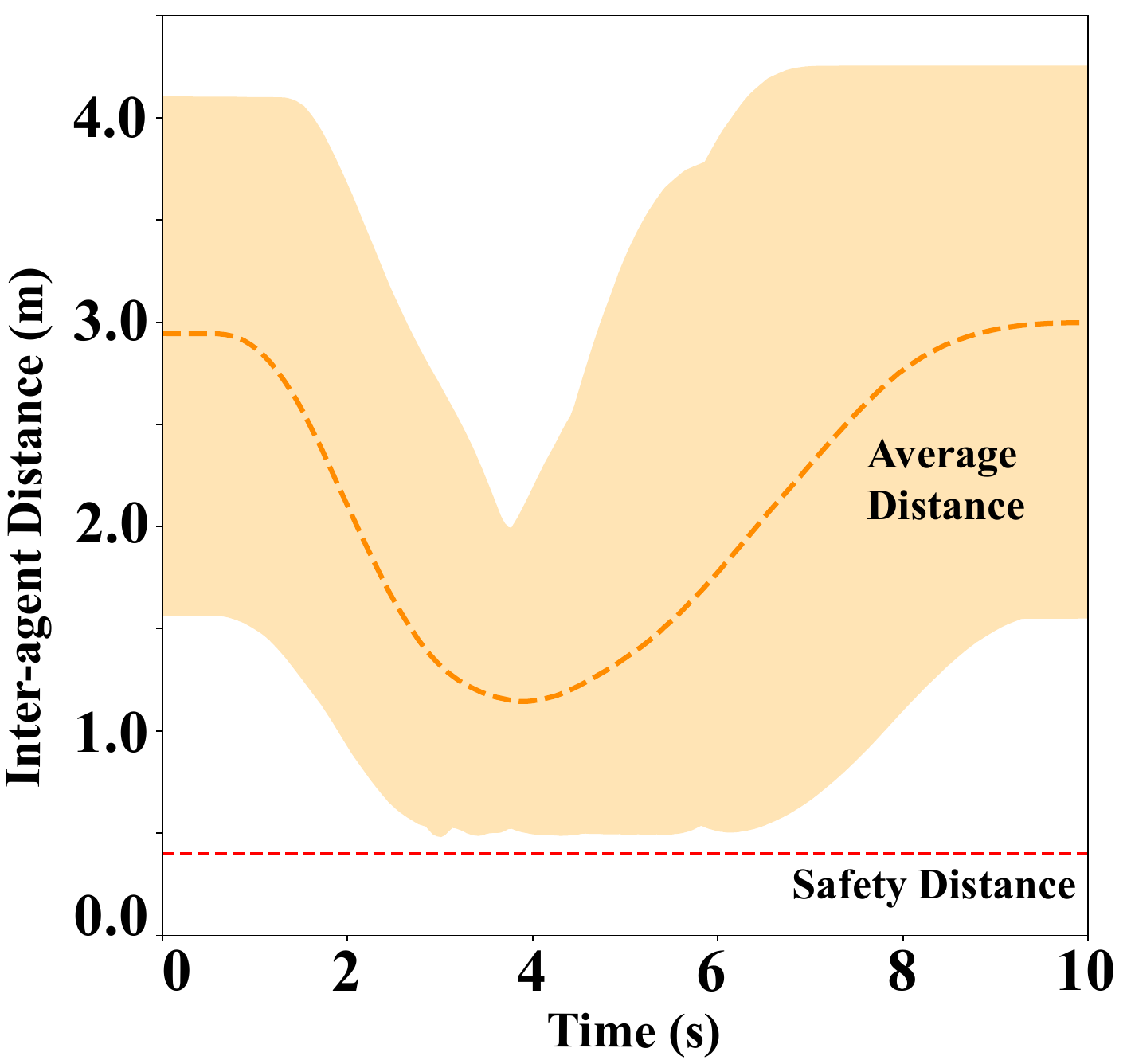}
	\caption{\textbf{Left}: The trajectories of the underlying agents. \textbf{Right}: The inter-agent distance.}
	\label{figure:smiu1}
	\vspace{-0.2in}
\end{figure}


\begin{figure} 
	\centering
	\includegraphics[width=0.6\linewidth]{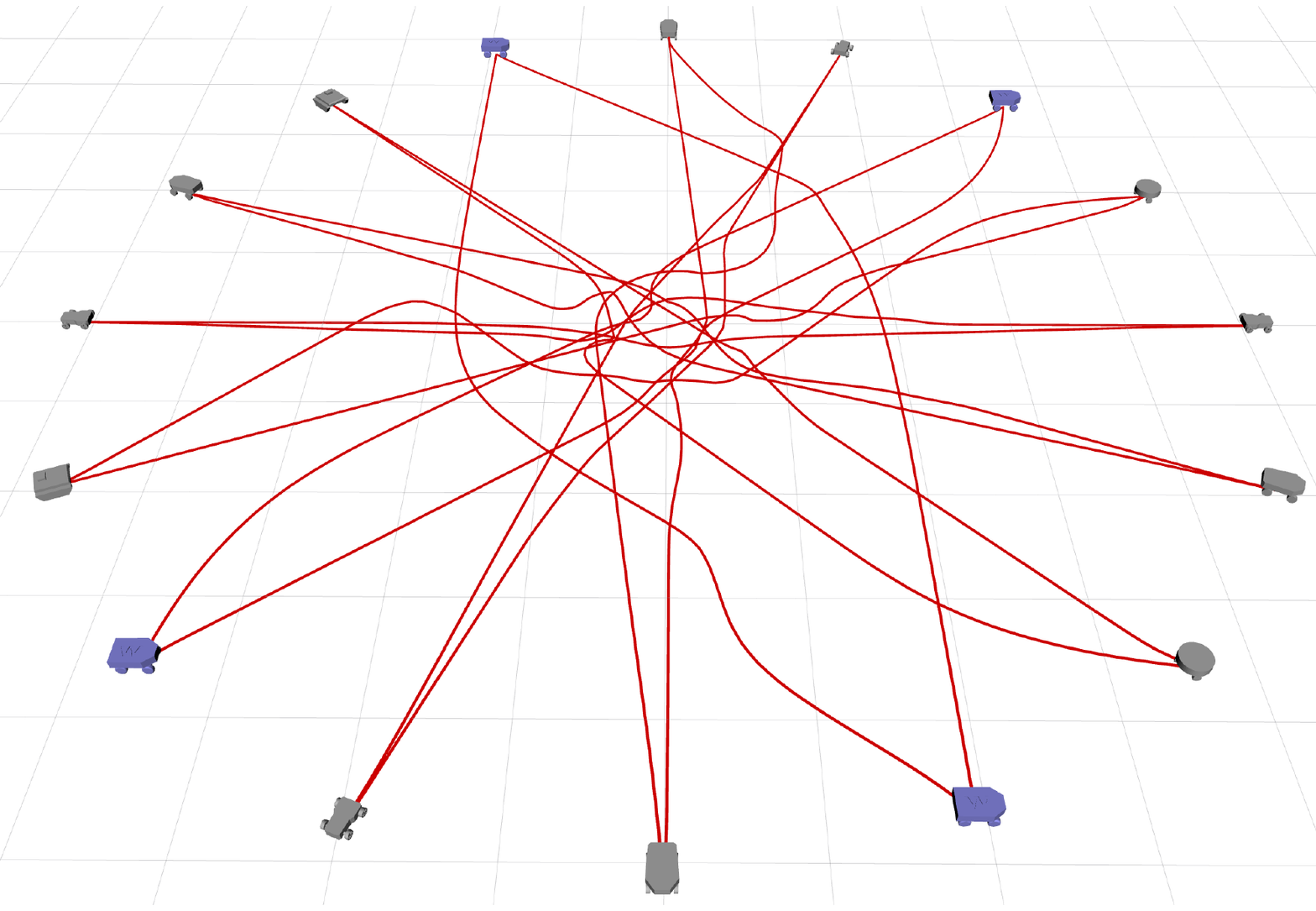}
	\includegraphics[width=0.9\linewidth]{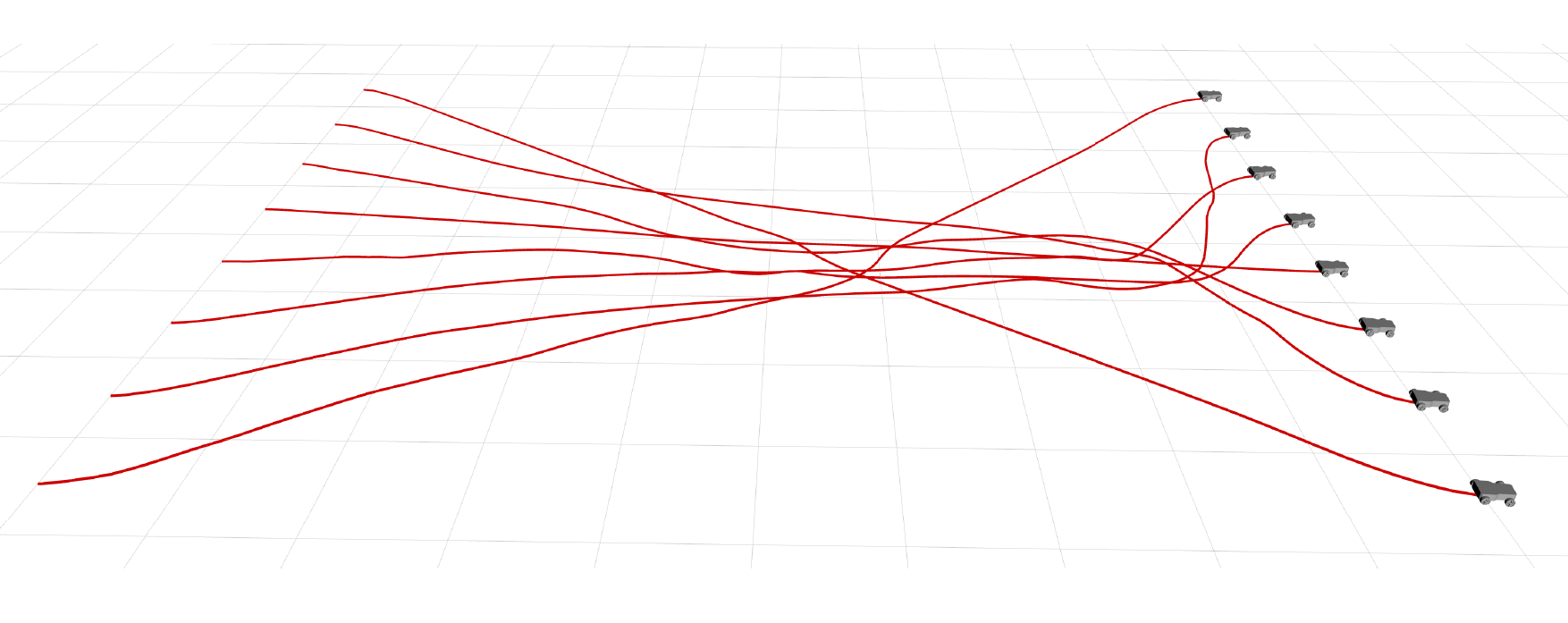}
	\caption{\textbf{Top}: The simulation of $16$ agents exchanging their positions. 
			\textbf{Bottom}: The simulation of $8$ agents changing lanes.  
		}
	\vspace{-0.3in}
	\label{figure:16agent}
\end{figure}

\begin{figure} 
	\centering
	\includegraphics[width=0.85\linewidth]{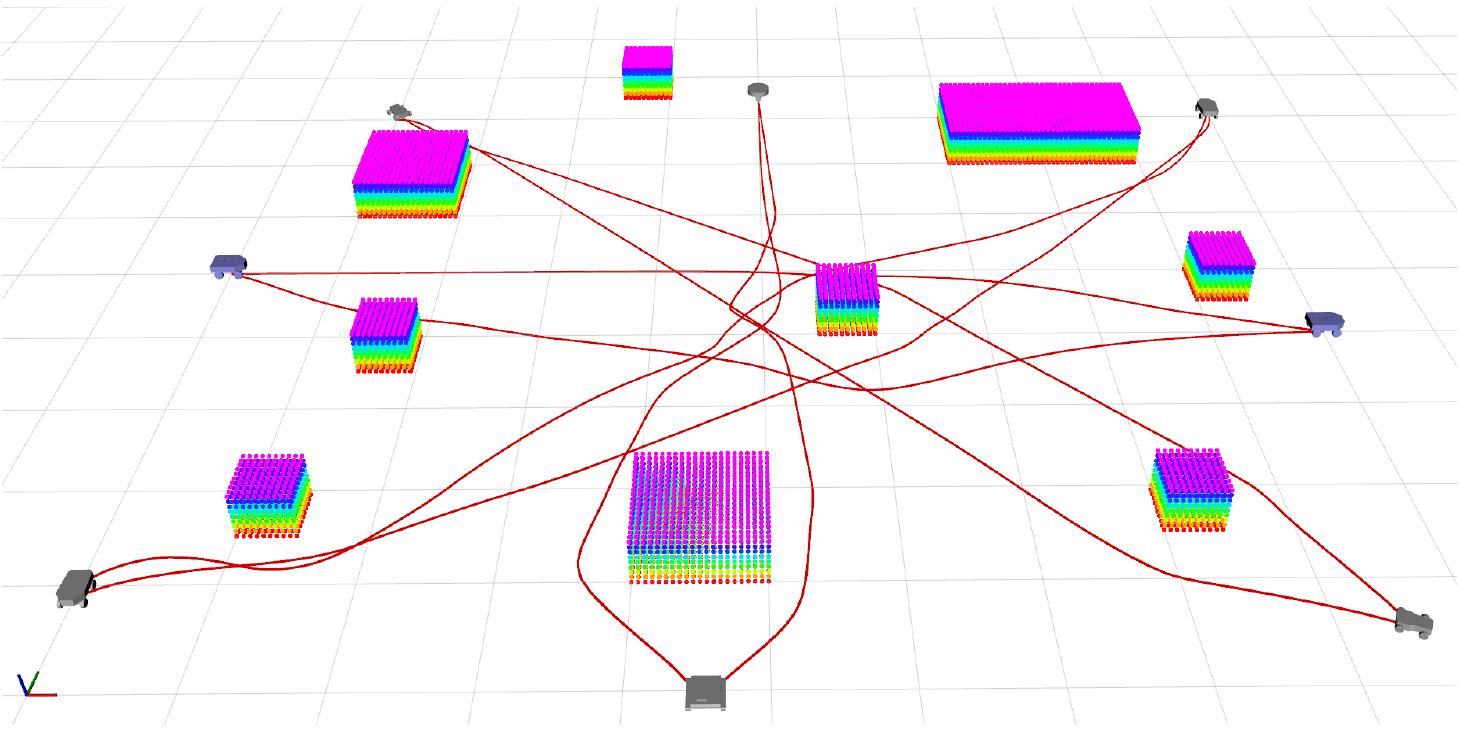}
	\caption{The simulation of 8 agents exchanging their positions in obstacle environment.}
	\vspace{-0.15in}
	\label{figure:8_in_obstacle}
\end{figure}

\begin{figure}
	\includegraphics[width=0.98\linewidth]{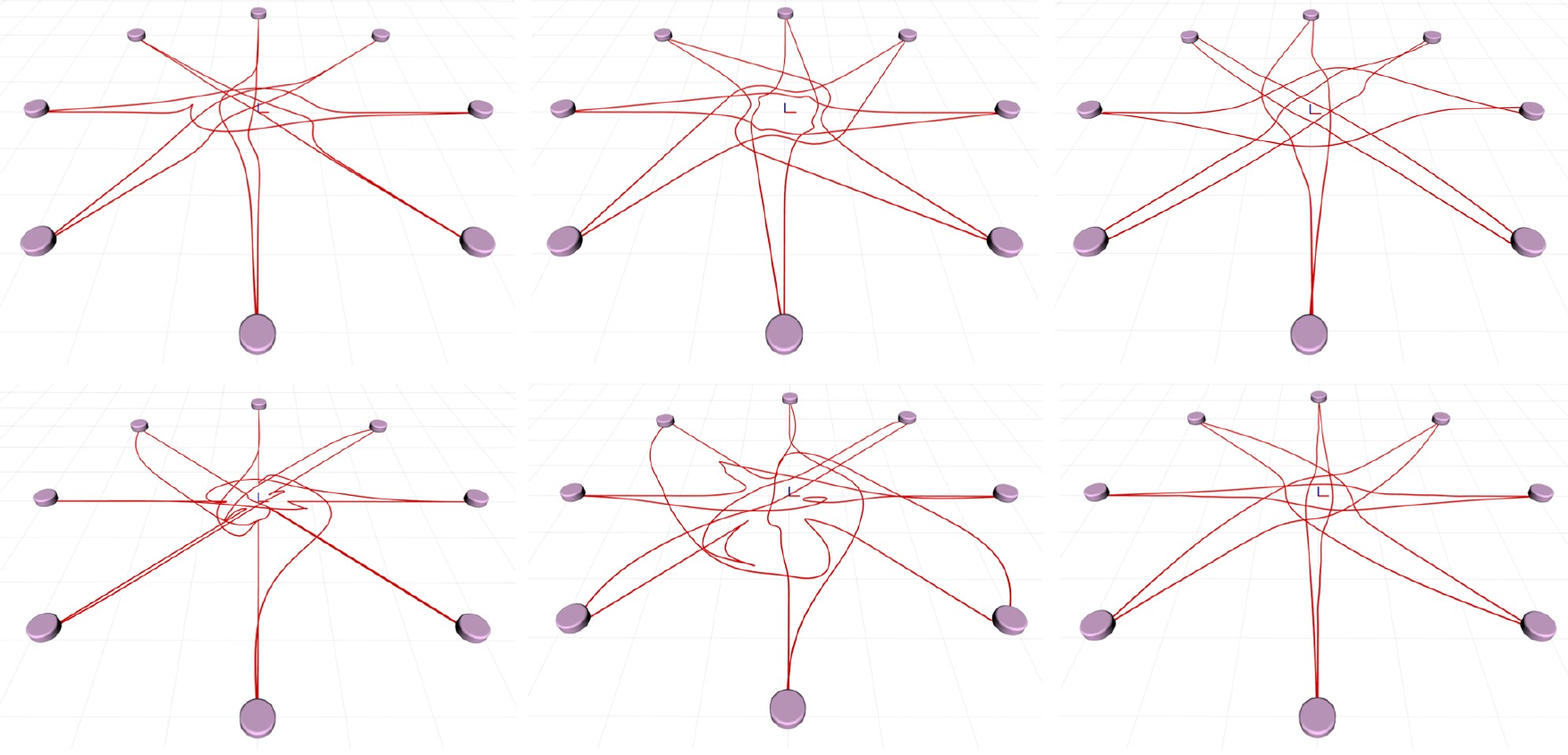}
	\caption{The results of motion planning. \textbf{Top-Left}: Ego-swarm \cite{Zhou2022}. \textbf{Top-Middle}: IMPC-DR \cite{Chen2022}. \textbf{Top-Right}: DMPC \cite{Luis2019}. \textbf{Bottom-Left}: LSC \cite{Park2022}. \textbf{Bottom-Middle}: MADER \cite{Jesus2021}. \textbf{Bottom-Right}: Proposed method.}
	\label{figure:compaison}
	\vspace{-0.2cm}
\end{figure}

\begin{table} [htbp]
	\caption{Comparisons. (${\rm Dis}$: distribution mode. $T_t[{\rm s}]$: moving time. $T_\text{cost}[{\rm ms}]$: replanning runtime. $L[{\rm m}]$: transition length. Result Representation: min/max value in the swarm.)}
	\label{table:comparison}
	\begin{tabular}{cccccccc} 
		\toprule
		Method & Dis.  & {$T_t$} & {$T_\text{cost}$}  & {$L$} \\
		\midrule
		Ego  \cite{Zhou2022}   & Seq.  & 10.0 / 11.6 & \textbf{3} / \textbf{11} & 8.33 / 8.80\\
		IMPC-DR \cite{Chen2022}   & Syn.  & 9.5 / 10.1 & 95 / 130 & 8.57 / 8.79 \\
		DMPC \cite{Luis2019}   & Syn.  & 9.4  / 10.6 & 51 / 90  & 8.43 / 8.89 \\
		LSC  \cite{Park2022}  & Syn.   & 11.1 / 12.6 & 15 / 46 & 9.11 / 9.73 \\
		MADER \cite{Jesus2021}  & Asyn. & 11.5 / 14.5 & 24 / 38 & 9.70 / 11.7\\
		Ours    & Asyn.  & \textbf{9.0}  / \textbf{9.6}  & 8 / 21 & \textbf{8.24} / \textbf{8.34}\\
		\bottomrule
	\end{tabular}
    \vspace{-0.1in}
\end{table} 

Then, we compare the proposed ASTA with five state-of-the-art methods \cite{Zhou2022,Chen2022,Luis2019,Park2022,Jesus2021}.
The dynamic models of the underlying robots are uniformly determined as double-integrators with maximum velocity and acceleration set as $1.0{\rm m/s}$ and $1.5{\rm m/s^2}$, respectively.
Furthermore, they uniformly have a planning horizon as $2.3{\rm s}$.
In this scenario, eight agents need to navigate to their antipodal positions without inter-agent collision.
The results are depicted in Fig.~\ref{figure:compaison} and provided in Table~\ref{table:comparison}.
Executed in a sequential (Seq.) way, the Ego-swarm \cite{Zhou2022} outperforms other methods in computing time and has a considerably good transition time and distance. 
The LSC \cite{Park2022}, IMPC-DR \cite{Chen2022} and DMPC \cite{Luis2019} all adopt synchronously (Syn.) concurrent replanning.
All of them can complete this task effectively, but with a relatively longer transition distance compared to Ego-Swarm.
In addition, replanning in an asynchronous (Asyn.) way, MADER in \cite{Jesus2021} has a relatively worse performance in this crowded scenario. 
This may be due to its check-recheck scheme, which results in more conservative trajectories.
In contrast, our method provides quicker and shorter transitions in addition to a relatively lower computational cost.
In summary, with additional capability to deal with the task in an asynchronous way, the proposed method not only outperforms its asynchronous counterpart in \cite{Jesus2021}, but also performs better than the others in \cite{Zhou2022,Chen2022,Luis2019,Park2022}.

\subsection{Hardware Experiments}

\begin{figure} 
	\centering
	\includegraphics[width=0.7\linewidth]{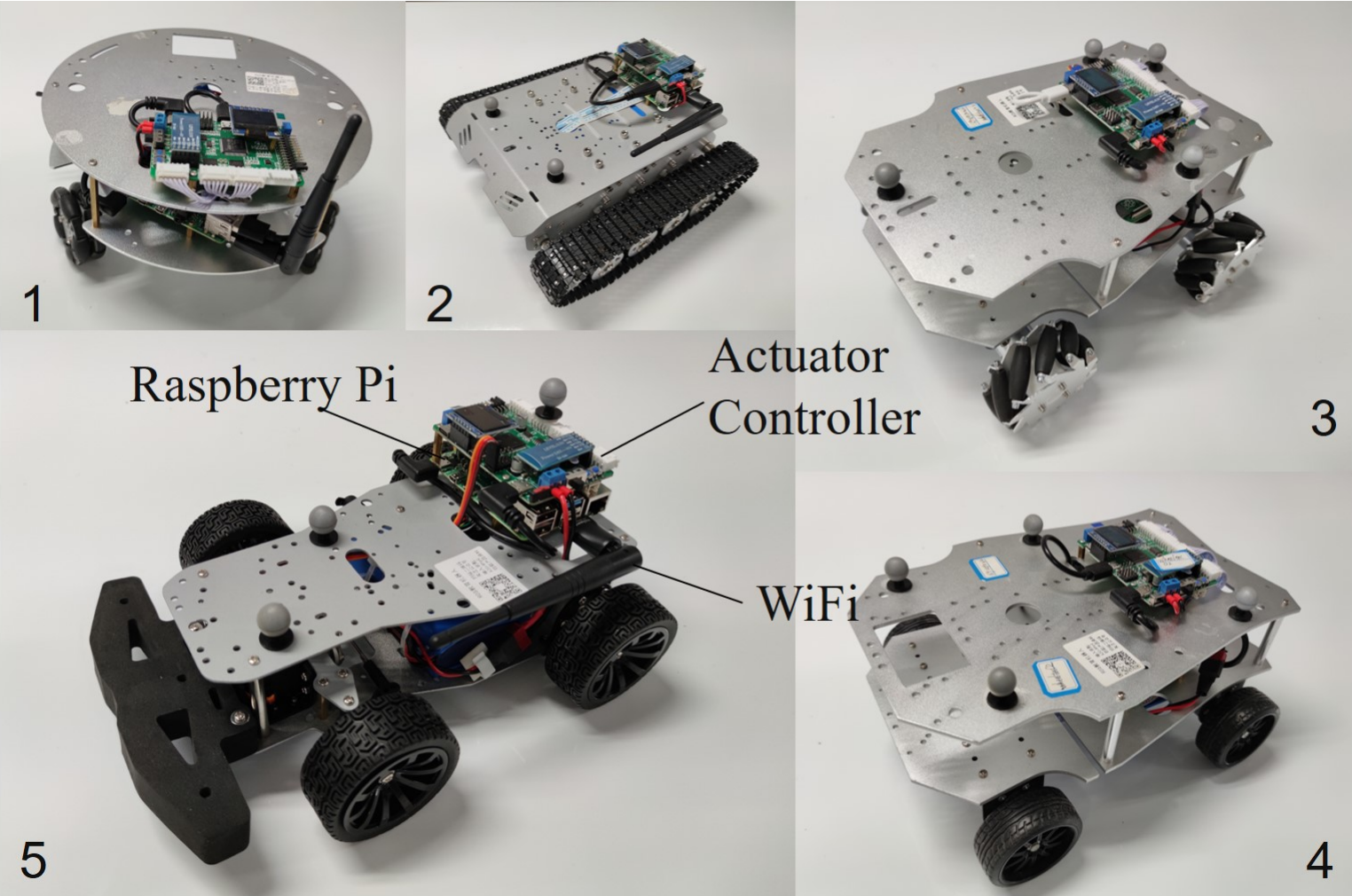}
	\includegraphics[width=0.65\linewidth]{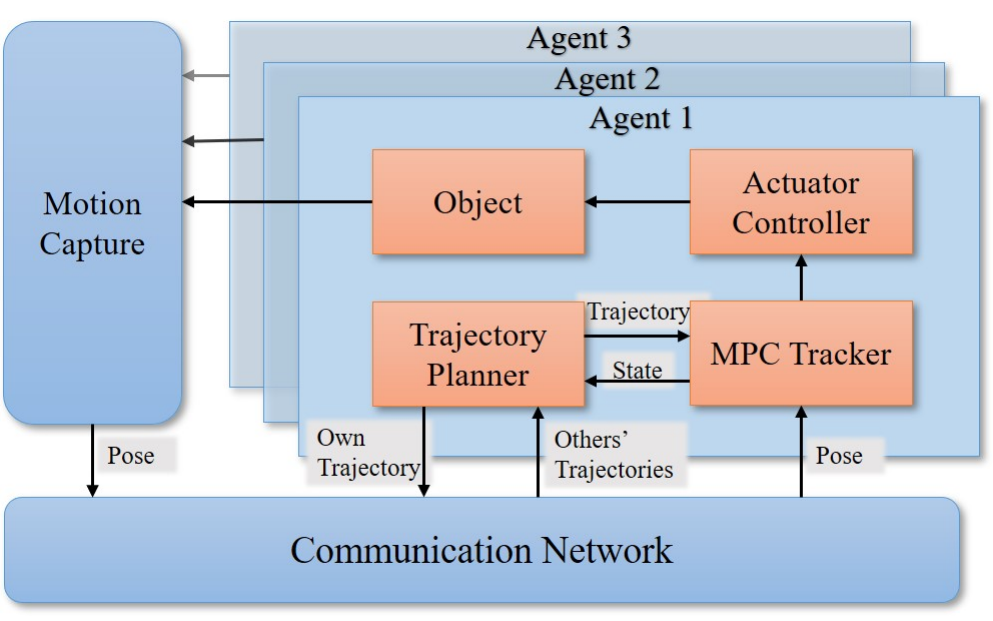}
	\caption{\textbf{Top}: The objects of hardware experiments. Five kinds of UGV are: 1) omni-directional car, 2) tank, 3) Mecanum car, 4) unicycle car, and 5) Ackermann car. \textbf{Bottom}: The system architecture of hardware experiments. }
	\label{figure:system}
\end{figure}

In this subsection, we verify the effectiveness of our method via hardware experiments.
In these experiments, $8$ UGVs of $5$ different hardware models are involved as shown in Fig.~\ref{figure:system}.
Their dynamic models include unicycles, double integrators, and kinematic bicycles.
The radii of these agents range from $0.15$m to $0.26$m, and their maximum velocities vary from $0.5{\rm m/s}$ to $1.0{\rm m/s}$.
The architecture of this multi-robot system is depicted in Fig.~\ref{figure:system}. 
The agents' positions are obtained via the indoor motion capture system OptiTrack and transmitted to other agents via WiFi and the ROS-based communication system.
The MPC-based trajectory tracker is utilized to track the planned trajectory whose results are sent to the actuator controller.
The computation of the trajectory planner, the MPC tracker and the actuator controller is accomplished by a Raspberry Pi 4B onboard computer.
To address the communication delays in experiments, we implemented a feedforward compensation mechanism to mitigate their impact.

The first scenario named ``crossing" lets eight UGVs cross in a $3.6{\rm m} \times 4.2 {\rm m}$ rectangle ground.
Since the maximum radius of the underlying agents reaches up to $0.26$m, 
this testing ground is thus crowded.
The actual moving snapshot is provided in Fig.~\ref{figure:antipodal-8}, where these agents can finish this navigation without any inter-agent collision. 
Furthermore, the average speed of all UGVs can reach up to $80\%$ of their corresponding maximum velocities, which demonstrates a considerably high agility.

\begin{figure}
	\centering
	\includegraphics[width=0.9\linewidth]{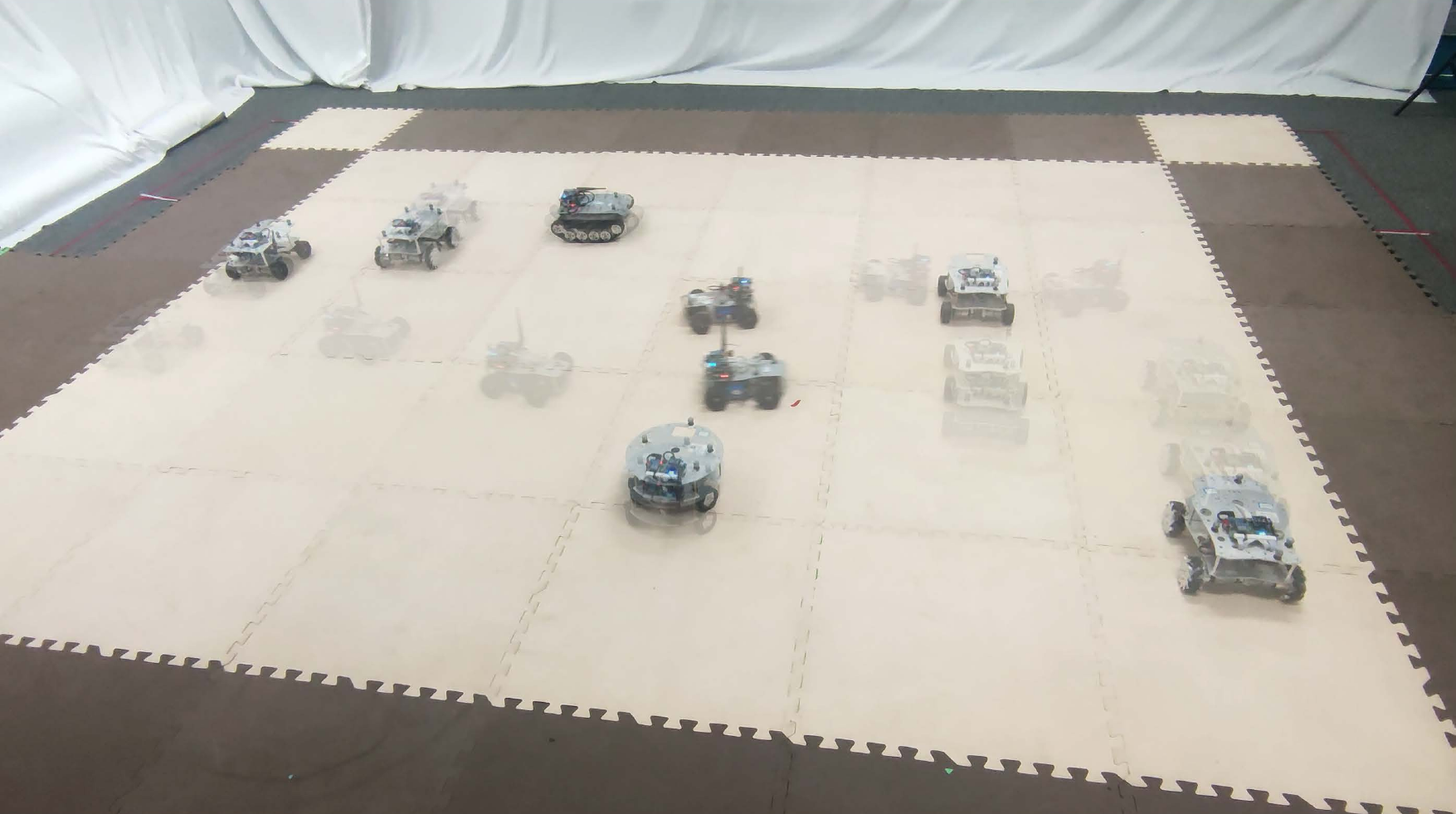}
	\caption{A snapshot of the experiment of un-signalized intersection.}
	\label{figure:crossing}
	\vspace{-0.22in}
\end{figure}

The second scenario is utilized to show the scalability of our method when an agent encounters new neighbors at an un-signalized intersection. 
Six agents carry on reciprocating motion in a specific area.
After tens of seconds, the other two Ackermann UGVs launch their trajectory planning procedures and set up communication with others.
Then, these two agents are sent to pass through this area. 
The result is illustrated in Fig.~\ref{figure:crossing}.
Two Ackermann UGVs only take less than $5$s to finish this task, which means that their average speeds can reach up to almost $0.95{\rm m/s}$, i.e., $95\%$ of the maximum velocity.
For more information about those experiments, please refer to the supplementary video.

\section{Conclusion}

This work has proposed a novel distributed asynchronous trajectory planning method called Asynchronous Spatial-Temporal Allocation for multi-agent systems with heterogeneous constrained dynamics.
Theoretical analysis, numerical simulations and hardware experiments all validate its effectiveness and performance.
Future work will extend the proposed method to address scenarios involving communication delays or malicious neighbors.

\appendices
\section{The Equivalence of Optimization Problems} \label{sec:appendix-1}
\begin{lemma} \label{lemma:finite-horizon}
	OCP~\eqref{original-optimal-control} with constraints~\eqref{original-obey-protocol} is equal to OCP \eqref{optimal-control}.
\end{lemma}

\begin{proof}
	To begin with, we will prove that a trajectory satisfying constraints~\eqref{finite-obey-protocl-1} and \eqref{finite-obey-protocl-2} can also satisfy constraints~\eqref{original-obey-protocol}. 
	For $t \in [t_{n^i}^i,t_{n^i}^i+T_{n^i}^i]$, if we have $S^i(x_{n^i}^i(t)) \in \mathcal{H}^{ij} (t)$, i.e., equation~\eqref{finite-obey-protocl-1},  then we can get that
	\begin{equation} \label{finite-horizon-protocol}
		\mathcal{T}raj_{n^i}^i (t_{n^i}^i,t_{n^i}^i+T_{n^i}^i) \subset \mathbb{A}_m^{ij} (t_{n^i}^i,t_{n^i}^i+T_{n^i}^i).
	\end{equation}
	
	If $t_e^{R_m} = t_{n^i}^i+T_{n^i}^i$, 
	for $t > t_{n^i}^i+T_{n^i}^i$, 
	$x_{n^i}^i(t) = x_{n^i}^i{( t_{n^i}^i+T_{n^i}^i)}$ and 
	$\mathcal{H}^{ji}(t) = \mathcal{H}^{ji}(t_e^{R_m}) = {\mathcal{H}^{ji} (t_{n^i}^i+T_{n^i}^i)}$ hold, then $S^i(x_{n^i}^i(t)) \in \mathcal{H}^{ij} (t)$ also holds, since ${S^i(x_{n^i}^i(t_{n^i}^i+T_{n^i}^i))} \in \mathcal{H}^{ij} (t_{n^i}^i+T_{n^i}^i)$.
	Combined with \eqref{finite-horizon-protocol}, we can obtain that 
	\begin{equation} \label{obey-protocol}
		\mathcal{T}raj_{n^i}^i (t_{n^i}^i,+\infty) \subset \mathbb{A}_m^{ij} (t_{n^i}^i,+\infty)
	\end{equation}
	which is equal to $\mathcal{T}raj_{n^i}^i \subset \mathbb{A}_m^{ij}$.
	
	Otherwise, we will consider the situation that $t_e^{R_m} > t_{n^i}^i+T_{n^i}^i$.
	Case 1): For $t > t_e^{R_m}$, $x_{n^i}^i(t) = x_{n^i}^i(t_e^{R_m})$ and $\mathcal{H}^{ji}(t) = \mathcal{H}^{ji}(t_e^{R_m})$ hold.
	Then, $S^i(x_{n^i}^i(t)) \in {\mathcal{H}^{ij} (t)}$ exists since $S^i(x_{n^i}^i(t)) = S^i(x_{n^i}^i(t_{n^i}^i+T_{n^i}^i)) \in \mathcal{H}^{ij} (t) = {\mathcal{H}^{ij} (t_{n^i}^i+T_{n^i}^i)}$.
	Case 2): For ${t \in (t_{n^i}^i+T_{n^i}^i,t_e^{R_m}]}$, we have enforced that $S^i(x_{n^i}^i(t)) = S^i(x_{n^i}^i(t_{n^i}^i+T_{n^i}^i)) \in \mathcal{H}^{ij} (t)$ in \eqref{finite-obey-protocl-2}.
	Combining with \eqref{finite-horizon-protocol}, we can also obtain \eqref{obey-protocol}.
	To sum up, a trajectory satisfying constraints~\eqref{finite-obey-protocl-1} and \eqref{finite-obey-protocl-2} also satisfies constraints~\eqref{original-obey-protocol}.
	
	Next we intend to prove that a trajectory satisfying constraints~\eqref{original-obey-protocol} can also satisfy constraints~\eqref{finite-obey-protocl-1} and \eqref{finite-obey-protocl-2}.
	At first, given condition that 
	$\mathcal{T}raj_{n^i}^i \subset \mathbb{A}_m^{ij}$, it is clear that equation~\eqref{finite-horizon-protocol} holds which means constraints~\eqref{finite-obey-protocl-1} are satisfied.
	For $t>t_{n^i}^i+T_{n^i}^i$, we have $x_{n^i}^i(t) = x_{n^i}^i(t_{n^i}^i+T_{n^i}^i)$ such that $x_{n^i}^i(t) \in \mathcal{H}^{ij}(t)$ can induce that 
	$x_{n^i}^i(t_{n^i}^i+T_{n^i}^i) \in \mathcal{H}^{ij}(t)$ holds for $t>t_{n^i}^i+T_{n^i}^i$.
	This can further imply constraints~\eqref{finite-obey-protocl-2}.
	As a result, the proof is completed.
\end{proof}

\section{Proof of Theorem~\ref{theorem:safety}}\label{sec:appendix-2}

The actual trajectory for agent $i$ during $[t_s^{R_m},t_s^{R_{m+1}})$ can be written as
\begin{equation} \label{trajecotry-between-m}
    \mathcal{T}raj^i(t_s^{R_m},t_s^{R_{m+1}}) = \bigcup_{n \in \mathcal{M}^i} \ \mathcal{T}raj_n^i ( t_a,t_b ),
\end{equation} 
where $t_a = \max\left\{ t_n^i,t_s^{R_m} \right\}$,
$t_b = \min \left\{ t_{n+1}^i,t_s^{R_{m+1}} \right\}$, and
\begin{equation}
\begin{aligned}
    \mathcal{M}^i \triangleq & \left\{ n \ | \ t_n^i \in [t_s^{R_m},t_s^{R_{m+1}}) \right\} \\
    & \cup \left\{ n \ | \ t_{n+1}^i \geq t_s^{R_m}, \ \ t_n^i < t_s^{R_m} \right\}.
\end{aligned}
\end{equation}
$\mathcal{M}^i$ represents the set of all the indices of the replanning steps whose trajectories are executed during $[t_s^{R_m},t_s^{R_{m+1}})$.

Considering that agents $i$ and $j$ obey their corresponding STA, 
we have $\mathcal{T}raj_{n^i}^i \subset \mathbb{A}_m^{ij}$ and  $\mathcal{T}raj_{n^j}^j \subset \mathbb{A}_m^{ji}$, where $n^i \in \mathcal{M}^i$ and $n^j \in \mathcal{M}^j$. 
According to Definition~\ref{def:STA},
we have 
$\mathcal{T}raj_{n^i}^i \cap \mathcal{T}raj_{n^j}^j = \emptyset$
and the following equation: 
\begin{equation*}
    \mathcal{T}raj_{n^i}^i(t_a^i,t_b^i ) \cap \mathcal{T}raj_{n^j}^j(t_a^j,t_b^j ) = \emptyset,
\end{equation*}
where $t_a^i = \max\left\{ t_{n^i}^i,t_s^{R_m} \right\}$, $t_b^i = \min \left\{ t_{n^i+1}^i,t_s^{R_{m+1}} \right\}$, and similar notations are used for agent $j$.
Incorporated with \eqref{trajecotry-between-m}, the above equation can be converted into
\begin{equation*}
    \mathcal{T}raj^i(t_s^{R_m},t_s^{R_{m+1}}) \cap \mathcal{T}raj^j(t_s^{R_m},t_s^{R_{m+1}}) = \emptyset.
\end{equation*}
This proves that the actual trajectories of agents $i$ and $j$ during time interval $[t_s^{R_m},t_s^{R_{m+1}})$ have no overlap.
Moreover, this property holds for $m = 1,2,\ldots, +\infty$, which means that the inter-agent collision can be avoided after $t_s^{R_1}$. 

Then, we discuss the time interval $[t^{R_1},t_s^{R_1}]$.
Note that no matter how STA is updated, STA in this interval is fixed as 
$\mathbb{A}_m^{ij}(t^{R_1},t_s^{R_1}) = \mathcal{R}_1^{ij}(t^{R_1},t_s^{R_1})$ ($m = 1,2,\ldots$).
Thus, for agents $i$ and $j$, any of their planned trajectories satisfies $\mathcal{R}_1^{ij}(t^{R_1},t_s^{R_1})$ and $\mathcal{R}_1^{ji}(t^{R_1},t_s^{R_1})$, respectively, thus is collision-free.
In particular, they are collision-free inherently when reaching their first renewal,  because of the precondition.
Therefore, in a similar way as above we can prove that  $\mathcal{T}raj^i(t^{R_1},t_s^{R_1}) \cap \mathcal{T}raj^j(t^{R_1},t_s^{R_1}) = \emptyset$.
Then, the proof is completed.

\begin{figure*}[t]
	\centering
	\includegraphics[width=0.75\linewidth]{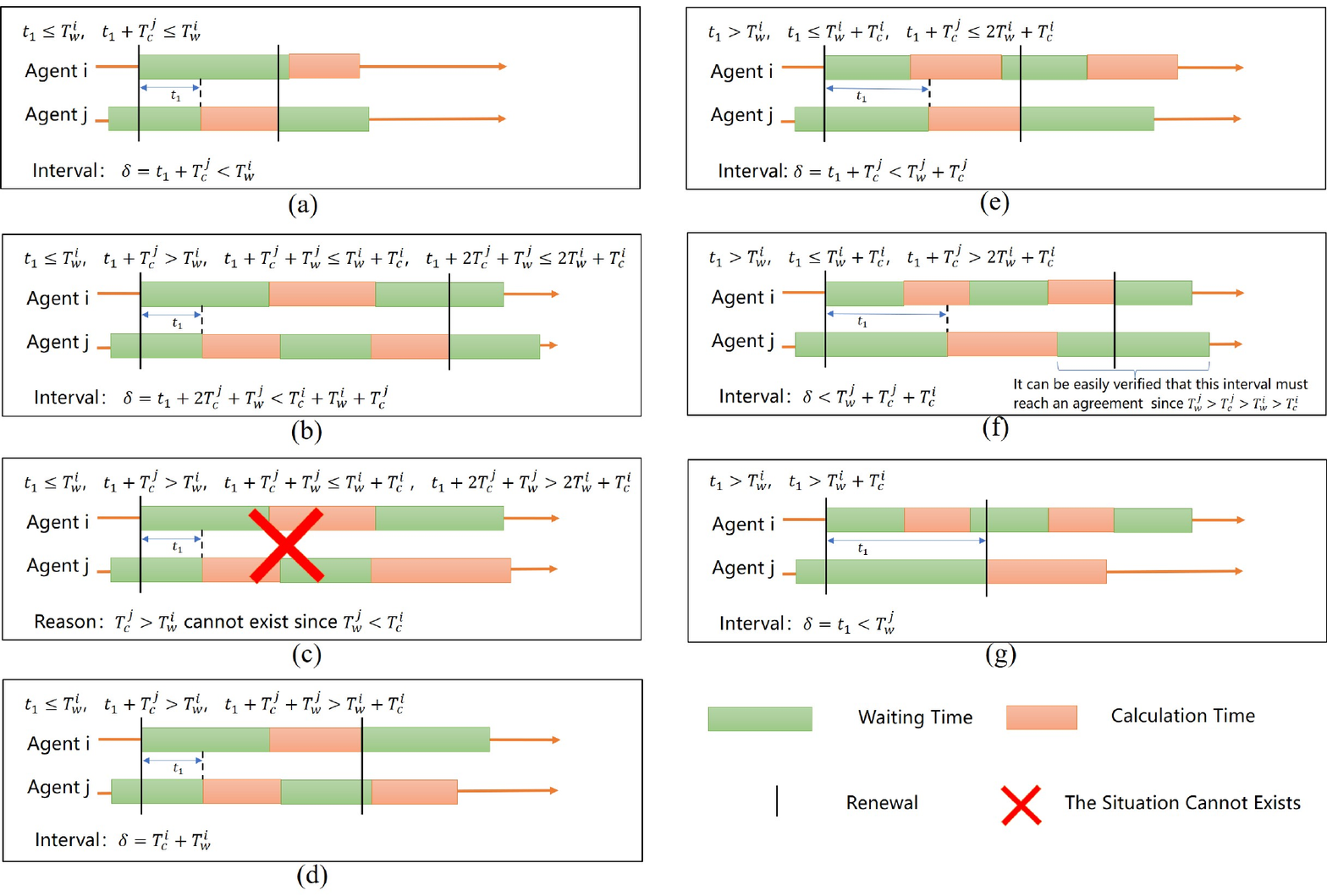}
	\caption{ In this figure, we illustrate all the possible cases of two consecutive allocation update which is dedicated for the proof of Theorem~\ref{theorem:lower-bound-frequency}. For brevity, we rewrite calculation time as $T_c^i$ and waiting time as $T_w^i$.}
	\label{figure:agreement-interval}
    \vspace{-0.2in}
\end{figure*}

\section{Proof of Theorem~\ref{theorem:lower-bound-frequency}}\label{sec:appendix-3}

We list all the possible situations in Fig.~\ref{figure:agreement-interval} to investigate the time interval $\delta$ between two adjacent updates.
Without loss of generality, assume that the first renewal is reached after agent $i$ finishes its calculation.
Accordingly, $t_1$ is defined as the time interval between their first renewal and the ending time of agent $j$'s first waiting time.
Since the situations enumerated in Fig.~\ref{figure:agreement-interval} all have the exact expression of time interval $\delta$ except for cases (c) and (f), we will pay special attention to these two scenarios.

Regarding case~(c), since the assumption $T_w^i>T_c^j$ in Section~\ref{subsection:protocol}, it can be easily verified that this case cannot exist.
Towards case~(f), since $T_w^j>T_c^j>T_w^i>T_c^i$, it can be verified that $\delta<T_c^j+T_w^j+T_c^i$.

After analyzing all the situations, it is clear that
$\delta \leq \max \left\{ T_c^i + T_w^i, T_c^j + T_w^j \right\}+\min \left\{  T_c^i,  T_c^j \right\}$ holds and the lower bound of the updating frequency can be easily obtained.

\bibliographystyle{IEEEtran}
\bibliography{REF}

\end{document}